\documentclass[letterpaper]{article}

\usepackage[utf8]{inputenc}
\usepackage{amsmath,amscd,amssymb,amsthm}
\newtheorem{theorem}{Theorem}

\newtheorem{corollary}{Corollary}
\usepackage{verbatim}
\usepackage{thmtools}
\usepackage{thm-restate}

\usepackage[margin=1in]{geometry}

\usepackage[framemethod=TikZ]{mdframed}
\mdfdefinestyle{MyFrame}{%
    linecolor=black,
    outerlinewidth=2pt,
    roundcorner=20pt,
    innerrightmargin=20pt,
    innerleftmargin=20pt,
    backgroundcolor=white}
\usepackage{natbib}
\usepackage{hyperref}

\newcommand{\R}{\mathbb{R}}
\newcommand{\E}{\mathop{\mathbb{E}}}
\newcommand{\argmin}{\mathop{\text{argmin}}}

\newcommand{\x}{x^\star}
\newcommand{\ol}{\mathcal{A}}

\usepackage{graphicx}

\usepackage{algorithm}
\usepackage{algorithmic}

\title{Anytime Online-to-Batch Conversions, Optimism, and Acceleration}

\author{
  \textbf{Ashok Cutkosky}\\Google\\\texttt{ashok@cutkosky.com}
}
\date{}

\usepackage{times}
\begin{document}

\maketitle

\begin{abstract}
A standard way to obtain convergence guarantees in stochastic convex optimization is to run an online learning algorithm and then output the average of its iterates: the actual iterates of the online learning algorithm do not come with individual guarantees. We close this gap by introducing a black-box modification to any online learning algorithm whose iterates converge to the optimum in stochastic scenarios. We then consider the case of smooth losses, and show that combining our approach with optimistic online learning algorithms immediately yields a fast convergence rate of $O(L/T^{3/2}+\sigma/\sqrt{T})$ on $L$-smooth problems with $\sigma^2$ variance in the gradients. Finally, we provide a reduction that converts any adaptive online algorithm into one that obtains the optimal accelerated rate of $\tilde O(L/T^2 + \sigma/\sqrt{T})$, while still maintaining $\tilde O(1/\sqrt{T})$ convergence in the non-smooth setting. Importantly, our algorithms adapt to $L$ and $\sigma$ automatically: they do not need to know either to obtain these rates.
\end{abstract}

\section{Online-to-Batch Conversions}

We consider convex stochastic optimization problems, where our objective is to minimize some convex function $\mathcal{L}:D\to \R$ where $D$ is some convex domain. We do not have true access to $\mathcal{L}$, however. Instead, we have a stochastic gradient oracle that given a point $x\in D$ will provide a random value $g$ such that $\E[g] = \nabla \mathcal{L}(x)$. Our objective is to use this noisy information to optimize $\mathcal{L}$.

A simple and extremely effective method for solving stochastic optimization problems is through online learning and online-to-batch conversion~\citep{shalev2011online, cesa2004generalization}. These techniques require remarkably few assumptions about the nature of the expected loss or the stochasticity in the system and yet still obtain optimal or near-optimal guarantees. This has helped fuel the widespread adoption of online learning algorithms as the method-of-choice in training machine learning models. Briefly, an online learning algorithm accepts a sequence of convex loss functions $\ell_1,\dots,\ell_T$ and outputs a sequence of iterates $w_1,\dots,w_T\in D$ where $D$ is some convex space and $w_t$ is output \emph{before} the algorithm observes $\ell_t$. Performance is measured by the regret:
\begin{align*}
    R_T(\x) = \sum_{t=1}^T \ell_t(w_t)-\ell_t(\x)
\end{align*}
A standard goal in online learning is to achieve \emph{sublinear regret}, which means that $\lim_{T\to\infty} R_T(\x)/T=0$. This indicates that the algorithm is doing just as well ``on average'' as the fixed benchmark point $\x$. In fact, most algorithms obtain non-asymptotic guarantees of the form $R_T(\x)=O(\sqrt{T})$, so that $R_T(\x)/T=O(1/\sqrt{T})$.


Online learning algorithms often adopt an adversarial model, in which no relationship is posited between $\ell_t$, but in our stochastic optimization problem we know that the $\ell_t$ are generated by some random process. This is where the Online-to-Batch conversion technique comes in~\citep{cesa2004generalization}. The classic argument is as follows: Set $\ell_t(x) = \langle g_t,x\rangle$ where $g_t$ is a stochastic gradient evaluated at $w_t$. Then observe $\mathcal{L}(w_t)-\mathcal{L}(\x)\le \E[\langle g_t,w_t-\x\rangle]$ and apply Jensen's inequality to obtain:
\begin{align*}
    \E\left[\mathcal{L}\left(\frac{\sum_{t=1}^T w_t}{T}\right)- \mathcal{L}(\x)\right]\le \frac{\E[R_T(\x)]}{T}
\end{align*}
We therefore output $\hat x =\tfrac{\sum_{t=1}^T w_t}{T}$ as an estimate of $\x$, and so long as the algorithm obtains sublinear regret, $\mathcal{L}(\hat x)-\mathcal{L}(\x)$ will approach zero in expectation. In fact, with $R_T(\x)=O(\sqrt{T})$, one obtains a convergence rate $O(1/\sqrt{T})$, which is often statistically optimal. 

One drawback of the online-to-batch conversion is that the iterates $w_t$ produced by the algorithm (where the noisy gradients are actually evaluated) do not necessarily converge to the optimal loss value. In fact, there is typically very little known about the behavior of any individual $w_t$. This is aesthetically unsatisfying and may even reduce performance. For example, \emph{optimistic} online algorithms can take advantage of stability in the gradients, performing well when $g_{t-1}\approx g_t$. We can hope for this behavior because intuitively the iterates should converge to $\x$ and so become closer together. Unfortunately, because actually we usually have few guarantees about the individual iterates $w_t$, it may not hold that $g_{t-1}\approx g_t$. We would like to make intuition match theory by enforcing some kind of stability in the iterates.

We address this problem by providing a black-box online-to-batch conversion: the iterates $x_t$ produced by our algorithm converge in the sense that $\mathcal{L}(x_t)\to\mathcal{L}(\x)$ (Section \ref{sec:otb}). We call this property \emph{anytime}, because the last iterate is always a good estimate of $\x$ at any time. Our reduction is quite simple, and bears strong similarity to the classical one. It stabilizes the iterates $x_t$, and we can exploit this stability when $\mathcal{L}$ is smooth. For example, when applied to an optimistic online algorithm, our reduction can leverage stability to improve the convergence rate on smooth losses from $O(L/T)$ to $O(L/T^{3/2})$ (Section \ref{sec:optimism}). Further, our reduction also has a surprising connection to the linear coupling framework for \emph{accelerated} algorithms \citep{allen2014linear}. We develop this connection to provide an algorithm that obtains a near-optimal (up to log factors) $\tilde O(L/T^2+\sigma/\sqrt{T})$ convergence rate for stochastic smooth losses with $\sigma^2=\text{Var}(g_t)$ without knowledge of $L$ or $\sigma$ while still guaranteeing $\tilde O(1/\sqrt{T})$ convergence rate for non-smooth losses (Section \ref{sec:acceleration}). In addition to these new algorithms, we feel that our analysis itself is interesting for its appealingly simplicity.

\subsection{Notation and Definitions}
We frequently use the compressed-sum notation $\alpha_{1:t} = \sum_{i=1}^t \alpha_i$ for any indexed variables $\alpha_t$. A convex function $f$ is $L$-smooth if $f(x+\delta)\le f(x) +\langle \nabla f(x), \delta\rangle + \tfrac{L}{2}\|\delta\|^2$ for and $x,\delta$, and $f$ is $\mu$ strongly convex if $f(x+\delta)\ge f(x) + \langle \nabla f(x),\delta\rangle +\tfrac{\mu}{2}\|\delta\|^2$ for all $x,\delta$. Given a convex function $f$ we say that $g$ is a subgradient of $f$ at $x$, or $g\in\partial f(x)$ if $f(y)\ge f(x) + \langle g,y-x\rangle$ for all $y$. $\nabla f(x)\in \partial f(x)$ if $f$ is differentiable.

\section{Anytime Online-to-Batch}\label{sec:otb}
In this section we provide our anytime online-to-batch conversion. Our algorithm is actually nearly identical to the classic online to batch: we set the $t$th iterate $x_t$ to be the average of the first $t$ iterates of some online learning algorithm $\ol$. The key difference is that we evaluate the stochastic gradient oracle at $x_t$, rather than the iterates provided by $\ol$. As a result, the outputs of $\ol$ in some sense exist only for analysis and are not directly visible outside the algorithm. Further, we incorporate \emph{weights} $\alpha_t$ into our conversion. Inspired by \citep{levy2017online}, these weights play a role in achieving faster rates on smooth losses, as well as removing log factors on strongly-convex losses. We provide specific pseudocode and analysis in Algorithm \ref{alg:aob} and Theorem \ref{thm:anytimeob} below.
 \begin{algorithm}\caption{Anytime Online-to-Batch}
   \label{alg:aob}
\begin{algorithmic}
   \STATE {\bfseries Input:} Online learning algorithms $\ol$ with convex domain $D$. Non-negative weights $\alpha_1,\dots,\alpha_T$ with $\alpha_1>0$.
   \STATE Get initial point $w_1\in D$ from $\ol$.
   \FOR{$t=1$ {\bfseries to} $T$}
   \STATE $x_{t} \gets \frac{\sum_{i=1}^t \alpha_i w_t}{\alpha_{1:t}}$.
   \STATE Play $x_t$, receive subgradient $g_t$.
   \STATE Send $\ell_t(x)=\langle \alpha_t g_t,x\rangle$ to $\ol$ as the $t$th loss.
   \STATE Get $w_{t+1}$ from $\ol$.
   \ENDFOR
   \RETURN $x_{T}$.
\end{algorithmic}
\end{algorithm}

\begin{theorem}\label{thm:anytimeob}
Suppose $g_1,\dots,g_t$ satisfy $\E[g_t|x_t]\in \partial \mathcal{L}(x_t)$ for some function $\mathcal{L}$ and $g_t$ is independent of all other quantities given $x_t$. Let $R_T(\x)$ be a bound on the linearized regret of $\ol$:
\begin{align*}
    R_T(\x) \ge \sum_{t=1}^T \langle \alpha_t g_t, w_t - \x\rangle 
\end{align*}
Then for all $\x\in D$, Algorithm \ref{alg:aob} guarantees:
\begin{align*}
    \E[\mathcal{L}(x_T) - \mathcal{L}(\x)]&\le \E\left[\frac{R_T(\x)}{\sum_{t=1}^T a_t}\right]
\end{align*}
Further, suppose that $D$ has diameter $B=\sup_{x,y\in D}\|x-y\|$ and $\|g_t\|_\star \le G$ with probability 1 for some $G$. Then with probability at least $1-\delta$,
\begin{align*}
    \mathcal{L}(x_T) - \mathcal{L}(\x) &\le\frac{R_T(\x)+2BG\sqrt{\sum_{t=1}^T \alpha_t^2 \log(2/\delta)}}{\sum_{t=1}^T \alpha_t}
\end{align*}
\end{theorem}

\begin{proof}
First, observe that
\begin{align*}
    \alpha_t(x_t - w_t)&=\alpha_{1:t-1}(x_{t-1} - x_t)
\end{align*}
where by mild abuse of notation we define $\alpha_{1:0}=0$ and let $x_0$ be an arbitrary element of $D$.

Now we use the standard convexity argument to say:
\begin{align*}
    \E\left[\sum_{t=1}^T \alpha_t(\mathcal{L}(x_t) -\mathcal{L}(\x))\right]&\le \E\left[\sum_{t=1}^T \alpha_t\langle g_t, x_t -\x\rangle\right]\\
    &=\E\left[\sum_{t=1}^T \alpha_t\langle g_t, x_t - w_t\rangle + \alpha_t\langle g_t, w_t - \x\rangle\right]\\
    &\le\E\left[R_T(\x)\right] + \E\left[\sum_{t=1}^Ta_{1:t-1} \langle g_t, x_{t-1} - x_t\rangle\right]
\end{align*}
Next we use convexity again to argue $\E[\langle g_t, x_{t-1} -x_t\rangle] \le \E[\mathcal{L}(x_{t-1}) - \mathcal{L}(x_t)]$, and then we subtract $\E[\sum_{t=1}^T \alpha_t \mathcal{L}(x_t)]$ from both sides:
\begin{align*}
    \E\left[\sum_{t=1}^T \alpha_t(\mathcal{L}(x_t) -\mathcal{L}(\x))\right]&\le \E\left[R_T(\x)\right] + \E\left[\sum_{t=1}^T\alpha_{1:t-1} (\mathcal{L}(x_{t-1}) - \mathcal{L}(x_t))\right]\\
    \E\left[-\alpha_{1:T} \mathcal{L}(\x)\right]&\le \E\left[R_T(\x)\right] +\E\left[\sum_{t=1}^T \alpha_{1:t-1}\mathcal{L}(x_{t-1}) -\alpha_{1:t}\mathcal{L}(x_t)\right]
\end{align*}
Finally, telescope the above sum to conclude:
\begin{align*}
    \E\left[\alpha_{1:T} \mathcal{L}(x_T) - \alpha_{1:T} \mathcal{L}(\x)\right] & \le \E\left[R_T(\x)\right]
\end{align*}
from which the in-expectation statement of the Theorem follows.

For the high-probability statement, let $H_{t-1}$ be the history $g_{t-1},x_{t-1},\dots,g_1,x_1$. Let $G_t = \E[g_t|H_{t-1}, x_t, w_t]$. Note that $G_t$ is still a random variable, and satisfies $G_t\in \partial \mathcal{L}(x_t)$. Next, let $\epsilon_t = \alpha_t\langle G_t, w_t -\x\rangle - \alpha_t \langle g_t, w_t - \x\rangle$. Then we have $\E[\epsilon_t|H_{t-1}, x_t, w_t]=0$ and:
\begin{align*}
    \sum_{t=1}^T \epsilon_t&=\sum_{t=1}^T \alpha_t \langle G_t, x_t -\x\rangle - \sum_{t=1}^T \alpha_t\langle g_t, x_t-\x\rangle\\
    |\epsilon_t|&\le 2\alpha_tBG\text{ with probability }1
\end{align*}
So by the Azuma-Hoeffding bound, with probability at least $1-\delta:$
\begin{align*}
    \sum_{t=1}^T \epsilon_t&\le 2BG\sqrt{\sum_{t=1}^T \alpha_t^2 \log(2/\delta)}
\end{align*}
Therefore with probability at least $1-\delta$, we have
\begin{align*}
    \sum_{t=1}^T \alpha_t(\mathcal{L}(x_t) - \mathcal{L}(\x))&\le \sum_{t=1}^T \alpha_t\langle G_t, x_t - \x\rangle\\
    &\le \sum_{t=1}^T \alpha_t\langle G_t, x_t - w_t\rangle + \sum_{t=1}^T \alpha_t\langle g_t, w_t - \x\rangle + \sum_{t=1}^T \epsilon_t\\
    &\le\sum_{t=1}^T \alpha_t\langle G_t, x_t - w_t\rangle + R_T(\x)+ 2BG\sqrt{\sum_{t=1}^T \alpha_t^2 \log\left(\frac{2}{\delta}\right)}
\end{align*}
Now an identical argument to the in-expectation part of the Theorem (but without need for taking expectations) yields:
\begin{align*}
    \mathcal{L}(x_T) - \mathcal{L}(\x) & \le \frac{R_T(\x)+2BG\sqrt{\sum_{t=1}^T \alpha_t^2 \log(2/\delta)}}{\sum_{t=1}^T \alpha_t}
\end{align*}

\end{proof}

As a corollary, we observe that the simple setting of $\alpha_t=1$ for all $T$ yields a direct analog of the classic online-to-batch conversion guarantee:
\begin{corollary}\label{thm:otb}
Under the assumptions of Theorem \ref{thm:anytimeob}, set $\alpha_t=1$ for all $t$. Then $R_T(\x)=\sum_{t=1}^T \langle g_t, w_t-\x\rangle$, which is the usual un-weighted regret. We have
\begin{align*}
    \E[\mathcal{L}(x_T) - \mathcal{L}(\x)]&\le \E\left[\frac{R_T(\x)}{T}\right]
\end{align*}
Further, $x_T = \frac{1}{T} \sum_{t=1}^T w_t$.
\end{corollary}
Corollary  \ref{thm:otb} is quite similar to the classic online-to-batch conversion result: in both cases, the average of the online learner's predictions has excess loss bounded by the average regret. Again, the critical difference is that in Algorithm \ref{alg:aob}, the actual outputs where the gradients are evaluated are the averaged outputs of the online learner. Thus the loss of the iterates converges to the minimum loss for Algorithm \ref{alg:aob}, which is not the case for the standard reduction.

In addition to this anytime online-to-batch result, we show below that Algorithm \ref{alg:aob} also maintains low regret:
\begin{corollary}\label{thm:regret}
Under the assumptions of Theorem \ref{thm:anytimeob}, let $R^M(\x)\ge \max_t R_t(\x)$. Then we have
\begin{align*}
\E\left[\sum_{t=1}^T \alpha_t (\mathcal{L}(x_t) -\mathcal{L}(\x))\right]&\le \E\left[R^M(\x)\left(1+\log\left(\frac{\alpha_{1:T}}{\alpha_1}\right)\right)\right]
\end{align*}
\end{corollary}
\begin{proof}
From Theorem \ref{thm:anytimeob} we have
\begin{align*}
\E[\alpha_t(\mathcal{L}(x_t) -\mathcal{L}(\x))]\le \E\left[\frac{\alpha_t R_t(\x)}{\alpha_{1:t}}\right]\le \E\left[\frac{\alpha_t R^M_t(\x)}{\alpha_{1:t}}\right]
\end{align*}
Then observe that $\log(a)+b/(a+b)\le \log(a+b)$ and sum over $t$ to conclude the Corollary.
\end{proof}
Recall that essentially all online learning regret bounds are non-decreasing in $T$, so that $\max_t R_t(\x)=R_T(\x)$. Thus the regret of Algorithm \ref{alg:aob} is only a logarithmic factor worse than the regret of the original online learner. Moreover, in the typical case that $R_t(\x)=O(\sqrt{t})$, a trivial modification of the above proof shows that $\E[\mathcal{L}(x_T) -\mathcal{L}(\x)]\le O(1/\sqrt{T})$, so that in many cases one should not even incur the log factor.

In fact, the anytime result is significantly more powerful than a standard regret bound because it provides point-wise bounds. This allows us to achieve a variety of different weighted regret bounds \emph{simultaneously}:
\begin{corollary}\label{thm:simulweight}
Under the assumptions of Theorem \ref{thm:anytimeob}, further suppose that $R_T(\x)$ is non-decreasing in $T$ and set $\alpha_t=1$. Let $s_t = t^k$ for some constants $k> 0$ (note that Algorithm \ref{alg:aob} is not aware of $s_t$). Then
\begin{align*}
    \E\left[\frac{\sum_{t=1}^T s_t (\mathcal{L}(x_t) -\mathcal{L}(\x))}{s_{1:T}}\right]\le O\left(\frac{R_T(\x)}{T}\right)
\end{align*}
\end{corollary}
\begin{proof}
Observe $s_{1:t}=\Theta(t^{k+1})$ so that $\E[s_t (\mathcal{L}(x_t) -\mathcal{L}(\x))/s_{1:T}]\le O(\E[R_T(\x)]t^{k-1}/T^{k+1})$, and sum over $t$.
\end{proof}

\section{General Analysis}\label{sec:general}
In this section we provide a more general version of our online-to-batch reduction.
The previous analysis appears to critically rely on linearized regret $\E[\sum_{t=1}^T \alpha_t(\mathcal{L}(x_t)-\mathcal{L}(\x))]\le \E[\sum_{t=1}^T \alpha_t\langle g_t, x_t-\x\rangle]$. This inequality may be tight for general convex losses, but in many cases we may want to take advantage of some known non-linearity in the losses. For example, when the loss function is $\mu$-strongly convex, one can use the inequality $\mathcal{L}(x_t)-\mathcal{L}(\x)\le \ell_t(x_t) - \ell_t(\x)$ where $\ell_t(x) = \langle \nabla \mathcal{L},x\rangle + \frac{\mu}{2}\|x-x_t\|^2$, leading to a $O(\log(T)/T)$ convergence rate rather than $O(1/\sqrt{T})$ \citep{hazan2007logarithmic}. In order to incorporate this information in our framework, we propose Algorithm \ref{alg:gaob}.

Algorithm \ref{alg:gaob} modifies Algorithm \ref{alg:aob} by considering an oracle that produces losses $\ell_t$ rather than stochastic gradients $g_t$. Specifically, we will require $\ell_t$ that are convex and lower-bound $\mathcal{L}$ in expectation. This generalizes the linear losses of Algorithm \ref{alg:aob}, and it may often be possible to construct nonlinear $\ell_t$ via only a gradient oracle, such as in the strongly-convex case. Our strategy for using these losses is essentially unchanged from that of Algorithm \ref{alg:aob}, but now our analysis is slightly more delicate since we cannot exploit the nice algebraic properties of linearity.

 \begin{algorithm}
   \caption{General Anytime Online-to-Batch}
   \label{alg:gaob}
\begin{algorithmic}
   \STATE {\bfseries Input:} Online learning algorithms $\ol$ with convex domain $D$. Non-negative weights $\alpha_1,\dots,\alpha_T$ with $\alpha_1>0$
   \STATE Get initial point $w_1\in D$ from $\ol$.
   \FOR{$t=1$ {\bfseries to} $T$}
   \STATE $x_t \gets \frac{\sum_{i=1}^t \alpha_iw_i}{\alpha_{1:t}}$
   \STATE Play $x_t$, compute loss $\ell_t$.
   \STATE Send $\alpha_t\ell_t(x)$ to $\ol$ as the $t$th  loss.
   \STATE Get $w_{t+1}$ from $\ol$.
   \ENDFOR
   \RETURN $x_{T}$.
\end{algorithmic}
\end{algorithm}
\begin{theorem}\label{thm:gotb}
Suppose $\ell_t$ is convex and satisfies $\mathcal{L}(x_t)-\mathcal{L}(x)\le \E[\ell_t(x_t)-\ell_t(x)|x_t]$ for all $t$ and for all $x$. Then with
\begin{align*}
    R_T(\x) = \sum_{t=1}^T \alpha_t\ell_t(w_t)-\alpha_t\ell_t(\x_t)
\end{align*}
Algorithm \ref{alg:gaob} obtains
\begin{align*}
    \E[\mathcal{L}(x_T)-\mathcal{L}(\x)]&\le \E\left[\frac{R_T(\x)}{\sum_{t=1}^T \alpha_t}\right]
\end{align*}
\end{theorem}

\begin{proof}
\begin{align}
    \sum_{t=1}^T\alpha_t(\ell_t(x_t)-\ell_t(\x_t))&\le \sum_{t=1}^T \alpha_t(\ell_t(x_t)-\ell_t(w_t))+\sum_{t=1}^T \alpha_t(\ell_t(w_t)- \ell_t(\x_t))\nonumber\\
    &=R_T(\x)+\sum_{t=1}^T\alpha_t(\ell_t(x_t) - \ell_t(w_t))\label{eqn:nosub}
\end{align}
Now observe that $x_t = \frac{\alpha_{1:t-1} x_{t-1} + \alpha_t w_t}{\alpha_{1:t-1}}$. Therefore by Jensen's inequality we have
\begin{align*}
     \ell_t(x_t)&\le \frac{\alpha_{1:t-1}\ell_t(x_{t-1}) + \alpha_t\ell_t(w_t)}{\alpha_{1:t}}\\
    \alpha_t\ell_t(x_t)-\alpha_t\ell_t(w_t)&\le \alpha_{1:t-1}(\ell_t(x_{t-1})- \ell_t(x_t))
\end{align*}
Now plug this into (\ref{eqn:nosub}):
\begin{align*}
    \sum_{t=1}^T\alpha_t(\ell_t(x_t)- \ell_t(\x))&\le R_T(\x) + \sum_{t=1}^T \alpha_{1:t-1}\ell_t(x_{t-1}) - \alpha_{1:t-1}\ell_t(x_t)
\end{align*}
Now observe that $\E[\ell_t(x_{t-1}) - \ell_t(x_t)]\le \E[\mathcal{L}(x_{t-1})-\mathcal{L}(x_t)]$. So taking expectations yields:
\begin{align*}
    \E\left[\sum_{t=1}^T\alpha_t(\mathcal{L}(x_t)- \mathcal{L}(\x_t))\right]&\le \E\left[R_T(\x) + \sum_{t=1}^T \alpha_{1:t-1}(\mathcal{L}(x_{t-1})-\mathcal{L}(x_t))\right]
\end{align*}
Now the rest of the proof is identical to that of Theorem \ref{thm:anytimeob}.
\end{proof}

\subsection{Strongly Convex losses}
In this section we apply the more general Algorithm \ref{alg:gaob} to $\mu$-strongly-convex losses. We recover standard convergence rates using only a gradient oracle and knowledge of the strong-convexity parameter $\mu$. We note that similar results also apply to exp-concave losses or other cases with lower-bounded Hessians.
\begin{corollary}\label{thm:scexample}
Suppose $D$ has diameter $B$, $\|g_t\|\le G$ with probability 1, and $\ol$ is Follow-the-Leader: $w_{t+1}=\argmin \sum_{i=1}^t \ell_i(w)$. Suppose $\mathcal{L}$ is $\mu$-strongly convex and we set $\ell_t(x) = \langle g_t,x\rangle +\frac{\mu}{2}\|x-x_t\|^2$ where $\E[g_t|x_t]=\nabla \mathcal{L}(x_t)$. Let $\alpha_t = 1$ for all $t$. Then we have
\begin{align*}
    R_T(\x)\le \frac{(\mu B + G)^2(\log(T)+1)}{2\mu}
\end{align*}
and
\begin{align*}
    \E[\mathcal{L}(x_T) - \mathcal{L}(\x)]&\le \frac{(\mu B + G)^2(\log(T)+1)}{2\mu T}
\end{align*}
\end{corollary}
\begin{proof}
The fact that $\mathcal{L}(x_t)-\mathcal{L}(\x)\le \E[\ell_t(x_t)-\ell_t(\x)|x_t]$ follows from strong-convexity. Observe that $\|\nabla \ell_t(w_t)\|=\|g_t + \mu(w_t-x_t)\|\le G+\mu B$ so that $\ell_t$ is $G+\mu B$-Lipschitz. Then the bound on $R_T$ follows from standard analysis of the follow-the-leader algorithm using the fact that $\sum_{i=1}^t \ell_i(w)$ is $t\mu$-strongly convex \citep{mcmahan2014survey}:
\[
R_T(\x)\le \sum_{t=1}^T \frac{\|\nabla \ell_t(w_t)\|^2}{2t\mu}
\]
and then use $\sum_{i=1}^t 1/i \le \log(T)+1$.
\end{proof}
This corollary provides the anytime analog of the standard online-to-batch result for strongly-convex losses. However, it is well known that in the stochastic case the logarithmic factor is not necessary. Prior work has removed this via diverse mechanisms, including restarting schemes \citep{hazan2014beyond} and tail-averaging \citep{rakhlin2012making}. Here we show here that a simple modification of the weights $\alpha_t$ suffices to remove the log factors.\footnote{The same trick also works for standard Online-to-Batch.}
\begin{corollary}\label{thm:scnolog}
Under the assumptions of Corollary \ref{thm:scexample}, suppose that $\alpha_t=t$ for all $t$. Then we have
\begin{align*}
    R_T(\x) \le \frac{T(\mu B + G)^2}{\mu}
\end{align*}
and
\begin{align*}
    \E[\mathcal{L}(x_T) -\mathcal{L}(\x)]&\le \frac{2(\mu B + G)^2}{\mu (T+1)}
\end{align*}
\end{corollary}
\begin{proof}
In this case, $\alpha_t\ell_t$ is $t(\mu B + G)$-Lipschitz and $\sum_{i=1}^t \alpha_i\ell_i$ is $\alpha_{1:t}\mu=\frac{T(T+1)\mu}{2}$ strongly convex. Thus the regret of Follow-the-Leader is bounded by
\begin{align*}
R_T(\x)&\le \sum_{t=1}^T \frac{\|t\nabla \ell_t(w_t)\|^2}{t(t+1)\mu}\\
&\le T\frac{(\mu B + G)^2}{\mu}
\end{align*}
Now divide by $\alpha_{1:T}=T(T+1)/2$ to see the claim.
\end{proof}
\section{Adaptivity and Smoothness}\label{sec:smoothness}
Many so-called ``adaptive'' online algorithms obtain regret bounds of the form $R_T(\x) \le O\left( \psi(\x)\sqrt{\sum_{t=1}^T \|g_t\|^2}\right)$ for various functions $\psi$. For example, Mirror-Descent and FTRL-based algorithms often obtain $\psi(\x) = B$, where $B$ is the diameter of the space $D$ \citep{mcmahan2010adaptive, duchi10adagrad, hazan2008adaptive} while so-called ``parameter-free'' algorithms can obtain $\psi(\x) = \tilde O(\|\x\|)$,  providing optimal adaptivity to $\|\x\|$ at the expense of logarithmic factors \citep{cutkosky2018black}. These adaptive bounds can be shown to obtain the better regret guarantee $\E\left[\sum_{t=1}^T \mathcal{L}(w_t) - \mathcal{L}(\x)\right] \le O\left(L\psi(\x)^2 + \psi(\x)\sigma \sqrt{T}\right)$ when the loss $\mathcal{L}$ is $L$-smooth and $g_t$ has variance $\sigma$, by exploiting the self-bounding property $\|\nabla \mathcal{L}(x)\|^2 \le L(\mathcal{L}(x) -\mathcal{L}(\x))$ \citep{srebro2010smoothness, cutkosky2018distributed, levy2018online}.

The appealing property of this argument is that the algorithm knows neither $L$ nor $\sigma$ and yet automatically adapts to both parameters, matching the performance of an optimally-tuned SGD algorithm. Since Algorithm \ref{alg:aob} also obtains low regret, we can make a similar claim:
\begin{corollary}\label{thm:nonaccel}
Suppose $R_T(\x)\le \psi(\x)\sqrt{\sum_{t=1}^T \alpha_t^2 \|g_t\|^2}$. Suppose $\mathcal{L}$ is $L$-smooth and obtains its minimum at $\x\in D$. Suppose $g_t$ has variance at most $\sigma^2$. Then with $\alpha_t=1$ for all $t$, Algorithm \ref{alg:aob} obtains:
\begin{align*}
    \E[\mathcal{L}(x_T)-\mathcal{L}(\x)]&\le O\left(\frac{\psi(\x)^2L\log^2(T)}{T} + \frac{\sigma\log(T)}{\sqrt{T}}\right)
\end{align*}
\end{corollary}
\begin{proof}
Define $\Delta_t = \E[\mathcal{L}(x_t)-\mathcal{L}(\x)]$. Observe that \begin{align*}
    \E[\|\nabla g_t\|^2] \le\E[\|\nabla \mathcal{L}(x_t)\|^2]+\sigma^2\le L\Delta_t +\sigma^2
\end{align*}
\begin{align*}
    \E[R_T(\x)]\le \psi(\x)\sqrt{L\sum_{t=1}^T \Delta_t+T\sigma^2}
\end{align*}
Then apply Corollary \ref{thm:regret} and quadratic formula to obtain $\sum_{t=1}^T \Delta_t\le O\left(\psi(\x)^2L\log^2(T) + \sigma\log(T)\sqrt{T}\right)$ when $\alpha_t=1$ and observe $\Delta_T\le \E[R_T(\x)]/T$ to prove the Corollary.
\end{proof}

The assumption that $\x\in D$ and the log factors in this analysis are a bit troubling. By using weights $\alpha_t=t$ and careful analysis it may be possible to remove the log factors, but it is less clear how to easily deal with constrained domains. We will take a different path through optimism in the next section which will allow us to perform much better with much less effort.
\subsection{Optimism for Faster Rates}\label{sec:optimism}
In this section we show how to leverage our online-to-batch scheme in combination with \emph{optimistic} online learning to further speed up the convergence rate. We will achieve a rate of $O(L/T^{3/2} + \sigma/\sqrt{T})$ with no knowledge of either $L$ or $\sigma$, resulting in a kind of interpolation between the $O(L/T+ \sigma/\sqrt{T})$ rate and the optimal accelerated rate of $O(L/T^2 + \sigma/\sqrt{T})$ \citep{lan2012optimal}.

An optimistic online learning algorithm is an online learner that is given access to a series of ``hints'' $\hat g_1,\dots,\hat g_T$ where $\hat g_t$ is revealed to the learner after $g_{t-1}$ but \emph{before} it commits to $w_t$ \citep{hazan2010extracting, rakhlin2013online, chiang2012online, mohri2016accelerating}. Optimistic algorithms attempt to guarantee small regret when $\hat g_t\approx g_t$, because in this scenario the learner has a good guess for what the future will contain. In particular, the optimistic algorithm of \citep{mohri2016accelerating} guarantees regret:
\[
R_T(\x) \le B\sqrt{2\sum_{t=1}^T \alpha_t^2\|\hat g_t -g_t\|^2}
\]
where $B$ is the diameter of the $D$. A common choice for $\hat g_t$ is $g_{t-1}$. Intuitively, this choice is ``optimistic'' in the sense that we are hoping $g_{t-1}\approx g_t$, which is the case on smooth losses if the iterates are close together. Fortunately, it \emph{is} the case that $x_t$ is necessarily close to $x_{t-1}$, so we use this regret bound for faster convergence in Algorithm \ref{alg:opt} and Theorem \ref{thm:optimisticrates}.

 \begin{algorithm}
   \caption{Optimistic Anytime Online-to-Batch}
   \label{alg:opt}
\begin{algorithmic}
   \STATE {\bfseries Input:} Optimistic Online algorithm $\ol$ with domain $D$. Non-negative weights $\alpha_1,\dots,\alpha_T$ with $\alpha_1>0$.
   \STATE Get initial point $w_1\in D$ from $\ol$.
   \STATE Set $g_0=0$.
   \FOR{$t=1$ {\bfseries to} $T$}
   \STATE Send $\alpha_tg_{t-1}$ to $\ol$ ad $t$th hint. 
   \STATE $x_{t} \gets \frac{\sum_{i=1}^t \alpha_i w_t}{\alpha_{1:t}}$.
   \STATE Play $x_t$, receive subgradient $g_t$.
   \STATE Send $\ell_t(x)=\langle \alpha_t g_t,x\rangle$ to $\ol$ as the $t$th loss.
   \STATE Get $w_{t+1}$ from $\ol$.
   \ENDFOR
   \RETURN $x_{T}$.
\end{algorithmic}
\end{algorithm}

\begin{theorem}\label{thm:optimisticrates}
Suppose $D$ has diameter $B$ and $\ol$ obtains the regret bound $R_T(\x)\le B\sqrt{2\sum_{t=1}^T \alpha_t^2 \|\hat g_t-g_t\|^2}$ when given hints $\hat g_t$ ahead of the gradient $g_t$. Set $\alpha_t=t$ for all $t$. Suppose each $g_t$ has variance at most $\sigma^2$, and $\mathcal{L}$ is $L$-smooth. Then Algorithm \ref{alg:opt} yields:
\begin{align*}
    \E[\mathcal{L}(x_T)-\mathcal{L}(\x)]\le O\left(\frac{LB^2}{T^{3/2}} + \frac{\sigma B}{\sqrt{T}}\right)
\end{align*}
\end{theorem}
\begin{proof}
Since we set $\hat g_t=g_{t-1}$, the assumption on $\ol$ implies:
\[
R_T(\x) \le B\sqrt{2\sum_{t=1}^T \alpha_t^2\|g_{t-1} - g_t\|^2}
\]

We can write $g_t = \nabla \mathcal{L}(x_t) + \zeta_t$ where $\zeta_t$ is some mean-zero random variable with $\E[\|\zeta_t\|^2] \le \sigma^2$. Then by smoothness, for $t>1$  we have 
\begin{align*}
    \|g_t-g_{t-1}\| &\le \|\nabla \mathcal{L}(x_t) - \nabla \mathcal{L}(x_{t-1})\| + \|\zeta_t - \zeta_{t-1}\|\\
    &\le L\|x_t -x_{t-1}\| + \|\zeta_t -\zeta_{t-1}\|\\
    &\le \frac{L\alpha_tB}{\alpha_{1:t}} + \|\zeta_t\| + \|\zeta_{t-1}\|\\
    \E[\|\hat g_t- g_t\|^2] &\le 5\frac{L^2\alpha_t^2 B^2}{(\alpha_{1:t})^2} + 10\sigma^2
\end{align*}
where in the last step we used $(a+b+c)^2 \le 5(a^2+b^2+c^2)$. Further, for $t=1$, we have
\begin{align*}
    \E[\|g_1\|^2] &\le \E[(\|\nabla \mathcal{L}(x_1) - \nabla \mathcal{L}(\x)\| + \|\zeta_t\|)^2]\\
    \E[\|g_1-\hat g_1\|^2]&\le 2L^2B^2 + 2\sigma^2\le 5\frac{L^2B^2\alpha_1^2}{(\alpha_{1:1})^2}+10\sigma^2
\end{align*}
Next, observe that $\alpha_{1:t}>t^2/2$ so that 
\begin{align*}
    \E[\|\hat g_t- g_t\|^2] &\le 20\frac{L^2B^2}{t^2} + 10\sigma^2
\end{align*}
Now observe $\sum_{t=1}^T t^2<3(T+1)^3/2$ and apply Jensen:
\begin{align*}
    \E[R_T(\x)] &\le \E\left[B\sqrt{2\sum_{t=1}^T \alpha_t^2\|\hat g_t - g_t\|^2}\right]\\
    &\le B\sqrt{30(T+1)^3\sigma^2 + 40L^2B^2T}
\end{align*}
And by Theorem \ref{thm:anytimeob} we have the desired result:
\begin{align*}
    \E[\mathcal{L}(x_T)-\mathcal{L}(\x)]&\le \frac{4\sqrt{10}LB^2}{T^{3/2}} + \frac{4\sqrt{10}\sigma B}{\sqrt{T}}
\end{align*}
\end{proof}
Note that the ordinary online-to-batch conversion may not be able to obtain this rate: here we are critically relying on the stability of the iterates $x_t$ to guarantee that $g_t$ and $g_{t-1}$ are not too far apart, while in the standard online-to-batch conversion one would require stability in the $w_t$, which may not occur.

\subsection{Acceleration}\label{sec:acceleration}
In the \emph{deterministic} setting, \citep{levy2018online} showed how to use adaptive step-sizes in conjuction with the linear-coupling framework \citep{allen2014linear} to derive an accelerated algorithm that adapts to the smoothness parameter $L$. In this section we show that our Algorithm \ref{alg:aob} and analysis is actually very similar in spirit to the linear-coupling scheme and so we can also derive an accelerated algorithm that adapts to both smoothness \emph{and} variance optimally. To our knowledge this is the first accelerated algorithm to adapt to variance. Our analysis is arguably simpler than prior work: our proof is much shorter, we rely on only relatively simple properties of $\alpha_t$ and we do not use the internals of the online algorithm.

Unlike previously in this paper, but similar to \citep{levy2018online}, here we will require $\mathcal{L}$ to be defined on an entire vector space rather than potentially bounded domain $D$. We will also assume knowledge of some parameter $B$ such that $\|\x\|\le B/2$. Lifting these restrictions are both valuable future directions.

 \begin{algorithm}
   \caption{Adaptive Stochastic Acceleration}
   \label{alg:asa}
\begin{algorithmic}
   \STATE {\bfseries Input:} Bound $B\ge 2\|\x\|$, value $c$, Online learning algorithms $\ol$ with domain $D=\{\|w\|\le B/2\}$.
   \STATE Get initial point $w_1\in D$ from $\ol$.
   \STATE $y_0\gets w_1$.
   \FOR{$t=1$ {\bfseries to} $T$}
   \STATE $\alpha_t\gets t$.
   \STATE $\tau_t \gets \frac{\alpha_t}{\sum_{i=1}^t \alpha_i}$.
   \STATE $x_{t} \gets (1-\tau_t)y_{t-1} + \tau_t w_t$.
   \STATE Play $x_t$, receive subgradient $g_t$.
   \STATE $\eta_t\gets\frac{cB}{\sqrt{1+\sum_{i=1}^{t} \alpha_{1:i} \|g_i\|^2}}$
   \STATE $y_t\gets x_t - \eta_tg_t$.
   \STATE Send $\ell_t(x)=\langle \alpha_t g_t,x\rangle$ to $\ol$ as the $t$th loss.
   \STATE Get $w_{t+1}$ from $\ol$.
   \ENDFOR
   \RETURN $x_{T}$.
\end{algorithmic}
\end{algorithm}

\begin{theorem}\label{thm:accel}
Suppose $\E[g_t]=\nabla \mathcal{L}(x_t)$ for some $L$-smooth function $\mathcal{L}$ with domain an entire Hilbert space $H$. Suppose $\|g_t\|\le G$ with probability 1 and $g_t$ has variance at most $\sigma^2$ for all $t$. Suppose $\|\x\|\le B/2$. Let $D$ be the ball of radius $B/2$ in $H$ and suppose $\ol$ guarantees regret
\begin{align*}
    R_T(\x)&\le B\sqrt{2\sum_{t=1}^T \alpha_t \|g_t\|^2}
\end{align*}
Then with $c=2$, Algorithm \ref{alg:asa} guarantees:
\begin{align*}
    \E\left[\mathcal{L}(y_T)-\mathcal{L}(\x)\right]&\le  \frac{4B+8LB^2\log(1+G^2T^3)}{T^2} +\frac{4B\sigma\sqrt{\log(1+G^2T^3)}}{\sqrt{T}}
\end{align*}
\end{theorem}
\begin{proof}
The opening of our proof is again very similar to that of Theorem \ref{thm:anytimeob}: observe that
\begin{align*}
    \E\left[\sum_{t=1}^T \alpha_t(\mathcal{L}(x_t) -\mathcal{L}(\x))\right]& \le \E\left[R_T(\x)+ \sum_{t=1}^Ta_{1:t-1} \langle g_t, y_{t-1} - x_t\rangle\right]
\end{align*}
Next we use convexity again to argue $\E[\langle g_t, y_{t-1} -x_t\rangle ]\le \E[\mathcal{L}(y_{t-1}) - \mathcal{L}(x_t)]$, and then we subtract $\E[\sum_{t=1}^T \alpha_t \mathcal{L}(x_t)]$ from both sides:
\begin{align}
    \E[-\alpha_{1:T} \mathcal{L}(\x) ]&\le \E\left[R_T(\x) +\sum_{t=1}^T \alpha_{1:t-1}\mathcal{L}(y_{t-1}) -\alpha_{1:t}\mathcal{L}(x_t)\right]\label{eqn:midaccel}
\end{align}
Now we use smoothness to relate $\mathcal{L}(y_{t})$ to $\mathcal{L}(x_{t})$. Defining $\zeta_t = g_t -\nabla \mathcal{L}(x_t)$ and $\beta_t = \alpha_{1:t}$, we have:
\begin{align*}
    \E[\mathcal{L}(y_{t})]&\le \E[\mathcal{L}(x_{t}) + \nabla \mathcal{L}(x_{t})(y_{t}-x_{t}) + \frac{L}{2}\|x_{t}-y_{t}\|^2]\\
    &\le  \E\left[\mathcal{L}(x_{t}) - \eta_t\|g_t\|^2 + \eta_t \langle \zeta_t,g_t\rangle +\frac{L\eta_t^2\|g_t\|^2}{2}\right]
\end{align*}
Then multiply by $\beta_t$:
\begin{align*}
    \E[\beta_{t}(\mathcal{L}(y_{t}) - \mathcal{L}(x_{t}))]&\le \E\left[- \frac{cB\beta_t \|g_t\|^2}{\sqrt{1+\sum_{i=1}^{t} \beta_i \|g_i\|^2}} + \frac{L\beta_{t}\eta_t^2\|g_t\|^2}{2}+\beta_t\langle \zeta_t, g_t\rangle\right]
\end{align*}
Next, we borrow Lemma A.2 from \citep{levy2018online}: for positive numbers $x_1,\dots,x_n$
\begin{align*}
    \sqrt{\sum_{i=1}^n x_i}\le \sum_{i=1}^n \frac{x_i}{\sqrt{\sum_{i'=1}^i x_{i'}}}\le 2\sqrt{\sum_{i=1}^n x_i}
\end{align*}
Also, observe from convexity of $\log$ that:
\begin{align*}
    \sum_{i=1}^n \frac{x_i}{1+\sum_{i'=1}^i x_{i'}}\le \log\left(1+\sum_{i=1}^n x_i\right)
\end{align*}
Using this we obtain
\begin{align*}
    \E\left[\sum_{t=1}^T \beta_t (\mathcal{L}(y_{t}) - \mathcal{L}(x_{t}))\right]&\le \E\left[-cB \sqrt{1+\sum_{t=1}^T \beta_t\|g_t\|^2} +\frac{c^2B^2 L\log(1+G^2\beta_{1:T})}{2}\right.\\
    &\quad\quad\quad\left. +cB+\sum_{t=1}^T \langle \zeta_t, \beta_tg_t\rangle \eta_t\right]
\end{align*}
Using Cauchy-Schwarz we obtain:
\begin{align*}
    \E\left[\sum_{t=1}^T \langle \zeta_t, \beta_tg_t\rangle \eta_t\right]&\le \E\left[\sqrt{\sum_{t=1}^T \beta_t \|\zeta_t\|^2 }\sqrt{\sum_{t=1}^T \beta_t \|g_t\|^2\eta_t^2}\right]\\
    &\le \E\left[cB\sqrt{\sum_{t=1}^T \beta_t \|\zeta_t\|^2}\sqrt{ \log\left(1+\sum_{t=1}^t\beta_t\|g_t\|^2\right)}\right]\\
\end{align*}
And now use Jensen's inequality:
\begin{align*}
\E\left[\sum_{t=1}^T \langle \zeta_t, \beta_tg_t\rangle \eta_t\right]&\le \E\left[cB \sqrt{\sum_{t=1}^T \beta_t \|\zeta_t\|^2}\sqrt{\log(1+G^2\beta_{1:T})}\right]\\
    &\le cB\sigma \sqrt{\beta_{1:T}\log(1+G^2\beta_{1:T})}
\end{align*}
Where in the last line we observed $\E[\|\zeta_t\|^2]\le \sigma^2$.
Combining everything, we have
\begin{align*}
    \E\left[\sum_{t=1}^T -\alpha_t\mathcal{L}(\x)\right]&\le\E\left[ R_T(\x)+\sum_{t=1}^T \alpha_{1:t-1}\mathcal{L}(y_{t-1}) -\alpha_{1:t}\mathcal{L}(y_t)\right]\\
    &\quad +\E\left[\frac{c^2LB^2\log(1+G^2\beta_{1:T})}{2}-cB \sqrt{1+\sum_{t=1}^T \alpha_{1:t}\|g_t\|^2}\right.\\
    &\quad\quad\left.-cB+cB\sigma \sqrt{\beta_{1:t}\log(1+G^2\beta_{1:t})}\right]
\end{align*}

Now observe that $t^2>\alpha_{1:t}>\alpha_t^2/2$ and recall $R_T(\x)\le B\sqrt{2\sum_{t=1}^T \alpha_t^2 \|g_t\|^2}$. Therefore since $c=2$ we have:
\begin{align*}
    \E\left[R_T(\x)-cB \sqrt{1+\sum_{t=1}^T \alpha_{1:t}\|g_t\|^2}\right]\le \E\left[B\sqrt{2\sum_{t=1}^T \alpha_t^2 \|g_t\|^2}- 2 B \sqrt{\sum_{t=1}^T \alpha_t^2\|g_t\|^2/2}\right]\le 0
\end{align*}
Also, observe that $\beta_{1:T}\le\sum_{t=1}^T t^2 \le T^3$. Thus we telescope the sum to obtain:
\begin{align*}
\E[\alpha_{1:T}(\mathcal{L}(y_T)-\mathcal{L}(\x))]&\le cB+\frac{c^2B^2 L\log(1+G^2T^3)}{2}\\
&+ cBT^{3/2}\sigma\sqrt{\log(1+G^2T^3)}
\end{align*}
and dividing by $\alpha_{1:T}=\tfrac{T(T+1)}{2}$ completes the proof.
\end{proof}

We remark also that, similar to the algorithm of \citep{levy2018online}, our Algorithm \ref{alg:asa} is \emph{universal} in the sense that for non-smooth losses we recover the $O(1/\sqrt{T})$ rate with no modifications. In fact, our analysis improves somewhat over \citep{levy2018online} in that we maintain an adaptive convergence rate in the non-smooth setting.\footnote{We suspect this same adaptive non-smooth rate can be achieved by \citep{levy2018online} via similar improved analysis.}
\begin{theorem}\label{thm:universal}
Suppose $\E[g_t]=\mathcal{L}(x_t)$ for some convex function $\mathcal{L}$. Then Algorithm \ref{alg:asa} guarantees:
\begin{align*}
    \E[\left[\mathcal{L}(y_T)-\mathcal{L}(\x)\right]&\le \E\left[\tfrac{2R_T(\x) +B\sqrt{2\sum_{t=1}^T t^2 \|\nabla \mathcal{L}(y_t)\|^2}\sqrt{\log(1+G^3T^3)}}{T^2}\right]
\end{align*}
\end{theorem}
Note that in the setting with $\|g_t\|\le G$ and $R_T(\x)=O\left(\sqrt{\sum_{t=1}^T\alpha_t^2 \|g_t\|^2}\right)$, Theorem \ref{thm:universal} implies a convergence rate of $O(\sqrt{\log(T)}/T)$.
\begin{proof}
We start from (\ref{eqn:midaccel}), and again proceed to relate $\mathcal{L}(y_t)$ to $\mathcal{L}(x_t)$, this time without the aid of smoothness:
\begin{align*}
    \E[\mathcal{L}(y_t)-\mathcal{L}(x_t)]&\le \E[\langle \nabla \mathcal{L}(y_t), y_t-x_t\rangle]\\
    &\le \E[\|\nabla\mathcal{L}(y_t)\|\|g_t\|\eta_t]
\end{align*}
So by Cauchy-Schwarz, again defining $\beta_t=\alpha_{1:t}$ we have
\begin{align*}
\E\left[\sum_{t=1}^T \beta_t(\mathcal{L}(y_t)-\mathcal{L}(x_t))\right]&\le \E\left[\sum_{t=1}^T \beta_t\|\nabla\mathcal{L}(y_t)\|\|g_t\|\eta_t\right]\\
&\le \E\left[\sqrt{\sum_{t=1}^T \beta_t \|\nabla \mathcal{L}(y_t)\|^2]}\sqrt{\sum_{t=1}^T \beta_t\|g_t\|^2\eta_t^2}\right]\\
&\le \E\left[B \sqrt{\sum_{t=1}^T \beta_t \|\nabla \mathcal{L}(y_t)\|^2}\sqrt{\log(1+G^3T^3)}\right]
\end{align*}
And combining everything yields
\begin{align*}
    \E[-\alpha_{1:T} \mathcal{L}(\x) ]&\le \E\left[R_T(\x) +B\sqrt{\sum_{t=1}^T \beta_t \|\nabla \mathcal{L}(y_t)\|^2}\sqrt{\log(1+G^3T^3)}\right.\\
    &\quad\quad\left.+\sum_{t=1}^T \alpha_{1:t-1}\mathcal{L}(y_{t-1}) -\alpha_{1:t}\mathcal{L}(y_t)\right]
\end{align*}
Telescope the sum and rearrange to prove the theorem.
\end{proof}
\section{Conclusion}
We have provided a variant on the standard online-to-batch conversion technique that enables us to compute gradients at the iterates produced by the conversion algorithm rather than those produced by the online learning algorithm. This stabilizes the sequence of iterates and enables low regret even with respect to arbitrary polynomial weights. We show how to apply our approach to easily remove the log factors in stochastic strongly-convex optimization. Further, for smooth losses, we gain stability in the gradients which can be used by optimistic online algorithms. Finally, a small modification allows us to achieve the optimal stochastic accelerated rates. Not only is this the first method to adapt to both variance and smoothness optimally, it also is more general than prior analyses by virtue of being a black-box reduction from \emph{any} sufficiently adaptive online learning algorithm. Finally, a recent connection between optimism and acceleration by \cite{wang2018acceleration} suggests that it may be possible to improve our optimistic analysis further to match the optimal accelerated rate in an even simpler manner.
\bibliography{all}
\bibliographystyle{icml2019}


\end{document}